%% file: Main.tex
\documentclass[conference]{IEEEtran}

\usepackage{cite}
\usepackage[latin9]{inputenc}
\usepackage{amsthm}
\usepackage{amsmath}
\usepackage{amssymb}
\usepackage{amsfonts}
\usepackage{enumerate}
\usepackage{centernot}
\usepackage{multicol,comment}
\usepackage{xspace}
\usepackage{enumitem}
\usepackage{environ}
\usepackage{times}
\usepackage{float}
\usepackage{algorithmicx}
\usepackage{algpseudocode}
\usepackage{bm}
\usepackage{lscape}
\usepackage{pdflscape}
\usepackage{wrapfig}
\usepackage{rotating}
\usepackage{epstopdf}
\usepackage{setspace}
\usepackage{breqn}

\usepackage{epsfig}
\usepackage{xcolor}

\usepackage{algorithm}

\IEEEoverridecommandlockouts

\algnewcommand\algorithmicinput{\textbf{Input:}}
\algnewcommand\INPUT{\item[\algorithmicinput]}
\algnewcommand\algorithmicoutput{\textbf{Output:}}
\algnewcommand\OUTPUT{\item[\algorithmicoutput]}

\usepackage{graphicx,subcaption}

\usepackage[countmax]{subfloat}
\usepackage[unicode,pdfstartview=FitH]{hyperref}
\makeatletter

\newcommand*{\Resize}[2]{\resizebox{#1}{!}{$#2$}}%

\newcommand{\lyxmathsym}[1]{\ifmmode\begingroup\def\b@ld{bold}
	\text{\ifx\math@version\b@ld\bfseries\fi#1}\endgroup\else#1\fi}
\theoremstyle{plain}
\newtheorem{thm}{\protect\theoremname}

\newtheorem{lem}
{\protect\lemmaname}

{\protect\claimname}

\newtheorem{cor}
{\protect\corname}

\newtheorem{defin}
{\protect\defname}

\providecommand{\lemmaname}{Lemma}
\providecommand{\theoremname}{Theorem}
\providecommand{\claimname}{Claim}
\providecommand{\corname}{Corollary}
\providecommand{\defname}{Definition}

\makeatother

\begin{document}
\title{Learning Temporal Dependence from Time-Series Data with Latent Variables}


\author{
	\IEEEauthorblockN{Hossein Hosseini, Sreeram Kannan, Baosen Zhang and Radha Poovendran}
	Department of Electrical Engineering, University of Washington, Seattle, WA \\
	Email: \{hosseinh, ksreeram, zhangbao, rp3\}@uw.edu
	\thanks{This work was supported by ONR grants N00014-14-1-0029 and N00014-16-1-2710.}
}

\maketitle

\begin{abstract}
We consider the setting where a collection of time series, modeled as random processes, evolve in a causal manner, and one is interested in learning the graph governing the relationships of these processes. A special case of wide interest and applicability is the setting where the noise is Gaussian and relationships are Markov and linear. 

We study this setting with two additional features: firstly, each random process has a hidden (latent) state, which we use to model the internal memory possessed by the variables (similar to hidden Markov models). Secondly, each variable can depend on its latent memory state through a random lag (rather than a fixed lag), thus modeling memory recall with differing lags at distinct times. Under this setting, we develop an estimator and prove that under a genericity assumption, the parameters of the model can be learned consistently. We also propose a practical adaption of this estimator, which demonstrates significant performance gains in both synthetic and real-world datasets.

\end{abstract}


\input{Sections/Introduction}
\input{Sections/Related}

\input{Sections/Model}
\input{Sections/Equations}
\input{Sections/Learning}
\input{Sections/Simulations}
\input{Sections/Conclusion}

\bibliographystyle{ieeetr}
\bibliography{Main}

\input{Sections/AppendixA}

\end{document}

%% file: Sections/Introduction.tex
\section{Introduction}
As time series measurements become increasingly commonplace in many problems, developing algorithms that can learn the underlying structure and the relationships between the observed variables become necessary. An important class of such algorithms focuses on extracting the {\em linear} dependency of the observed variables; this line of work originated from the pioneering work \cite{granger1969investigating} proposing Granger causality. The linear temporal models have been used in numerous fields such as financial forecasting~\cite{cochrane2005time}, biological network modeling~\cite{lawrence2010learning}, and traditional control systems~\cite{young2012recursive}, because they are simple enough to {\em learn with limited data} and yet are effective in practice to model time series data. In these problems, the first step is to learn the model parameters, and then further questions about the system can be investigated, including prediction of future values, imputation of missing variables and causal inference.

\begin{figure}[!t]
	\centering
	\includegraphics[width=0.3\textwidth]{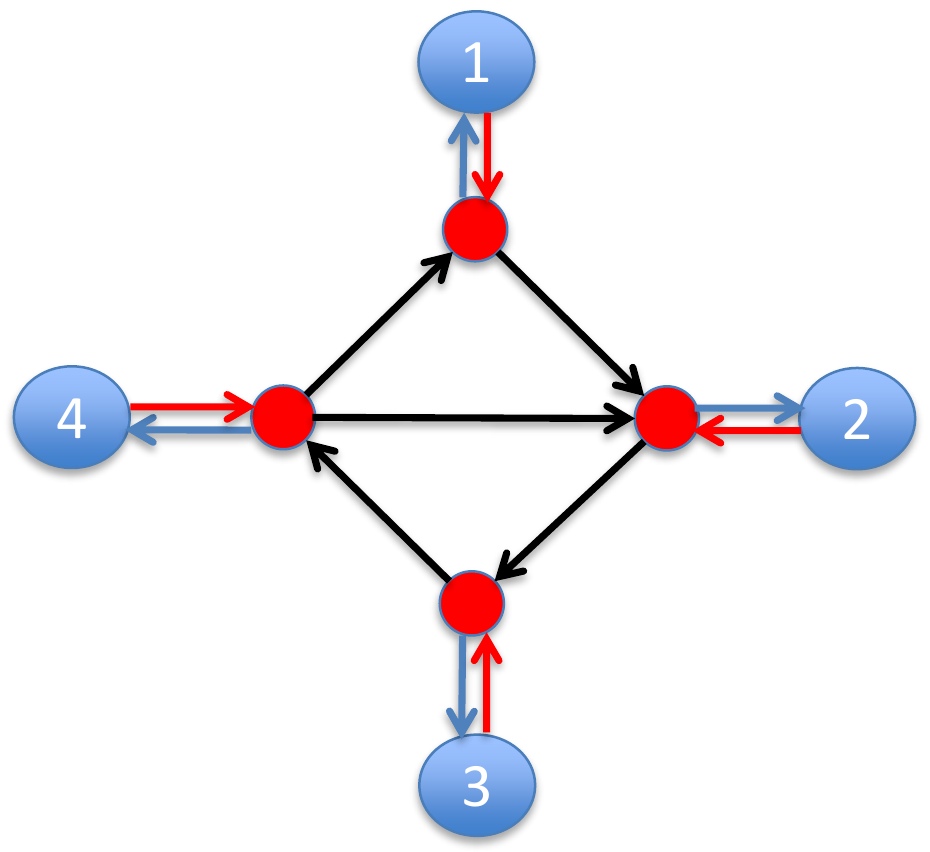}
	\caption{Example of Latent Temporal Model: The observed variables are shown in blue and the latent variables are shown in red. There is a sparse graph interconnecting the latent variables (edges shown in black). Also, each observed variable is influenced by its corresponding latent variable (edges shown in blue), and finally, each observed variable influences its latent variable (edges shown in red).}
	\label{fig:model}
\end{figure}

In many of the real-world datasets, some of the variables may be unobserved; most of the times, even the existence of such unobserved variables may be unknown. Therefore it is expedient to consider models which allow for some {\em hidden  or latent variables} and extract relationships not only between the observed variables but also between the latent variables. Inference in the presence of models with hidden variables has a long and distinguished history; a particular breakthrough is the work of \cite{dempster1977maximum} which proposed the Expectation-Maximization (EM) algorithm for maximizing the likelihood of observations in presence of latent variables. EM-based algorithms however do not guarantee convergence to the global optima of the likelihood landscape. 

In this paper, we unite the two threads by considering the learning of linear temporal relationships with latent variables. Our main contributions are the following.

\begin{enumerate}

\item We propose a new linear model for incorporating temporal latent variables, which captures the effects of the {\em temporal memory} in the system. Our proposed model has two important features:
\begin{itemize}
	\item For each observed variable, there is a latent component that acts as its memory,
	\item Each observed variable is affected by its memory with a random and time varying delay.
\end{itemize}

\item We provide the identifiability results for learning the system parameters using the observations.

\item We propose an efficient algorithm to learn the system parameters. We demonstrate that the proposed algorithm outperforms the baseline method for linear temporal models both in synthetic and real-world datasets.
\end{enumerate}

To illustrate the first aspect of our proposed model, consider an example where the variables are the disease states of various individuals over time, and we are interested in learning how the disease propagates through the network. In this context, it is unlikely that an individual who contacts a diseased individual immediately receives the disease; rather it may increase the {\em susceptibility} that can later manifest as the disease. Furthermore, susceptible individuals interact and influence each other during this incubation time period. In this case, susceptibility is an unobserved latent state, that can encode temporal memory inherent to the system. We refer the reader to Figure~\ref{fig:model} for an illustration of the model in a system with four observed time-series.

A second aspect of our model is that we allow the observed state to depend on a {\bf random lag} unobserved susceptibility, where the value of the random lag follows a certain distribution; this further enhances the temporal memory. The proposed model is also useful when the inter-sampling times are varying, un-precise, or unknown.

The rest of this paper is organized as follows. In Section II, we review the relevant literature. In Section III, we propose our model with latent variables and random lags. In Section IV, we derive linear equations, composed of covariances of observations and the model parameters. In Section V, we first present a theorem regarding the identifiability of the model parameters, and then propose an algorithm for learning the model parameters. Section VI provides the experimental results with synthetic data as well as real-world data, and Section VII concludes the paper.

%% file: Sections/Related.tex
\section{Related Work}


Sparse recovery from time series data have been developed through a long line of pioneering works (e.g.,~\cite{tibshirani1996regression,donoho2006compressed,wainwright2009sharp}), which are extensively used for Gaussian graphical model selection~\cite{meinshausen2006high,ravikumar2011high,friedman2008sparse,yuan2007model,banerjee2008model,danaher2014joint}.
A common theme of these works is that all variables involved in the model are observed.
%
%
In contrast, if there are \emph{latent variables} that cannot be directly observed, a naive model without considering them may result in a dense interaction graph with many spurious edges between the variables~\cite{loehlin1998latent}. The most common approach to this problem is developing methods based on the Expectation Maximization (EM)~\cite{dempster1977maximum, redner1984mixture}; however, because of non-convexity, EM suffers from the fact that it can get stuck in local optima and that it comes with weak theoretical guarantees.
In~\cite{chandrasekaran2010latent}, the authors proposed an alternative convex optimization approach to Gaussian graphical model selection with latent factors, assuming that i.i.d. samples are drawn from the distribution of observed variables, and that, compared to the number of observed variables, there are very few number of latent variables.

A somewhat parallel line of research on time series data is to identify \emph{causal relationships} between the variables. One of the earliest methods, also perhaps the most well-known, is the Granger causality~\cite{granger1969investigating}. Formally, $X$ Granger causes $Y$ if its past value can help to predict the future value of $Y$ beyond what could have been done with the past value of $Y$ only. Granger causality has been widely employed, due to its simplicity, robustness, and extendability, and many variants have been proposed for different application \cite{barrett2010multivariate,hiemstra1994testing,brovelli2004beta,marinazzo2008kernel,arnold2007temporal}. In~\cite{arnold2007temporal}, the authors applied Lasso regularization for graphical Granger model selection, and showed that it performs well in terms of accuracy and scalability. However, similar to the non-temporal models, when there are unseen/unknown time series, the simple temporal model can lead to wrong causal inference~\cite{spirtes2000causation}.  

In this paper, we study a model that combines both features of sparse latent variable recovery and temporal causality. Namely, we consider the problem of parameter estimation of a linear dynamical system with \emph{random delays} between the latent variables and observed variables. The proposed model is a generalization of the well-studied temporal Gaussian graphical model with hidden variables~\cite{ghahramani2001introduction}). In \cite{bengio1996input}, the state space model is generalized by including inputs. In~\cite{beal2005bayesian}, this model is revised in a sense that the previous observations are fed back into the model as inputs. In \cite{jalali2011learning}, the authors considered a first order vector autoregressive (VAR) model with hidden components, assuming that the number of hidden variables is {\em much fewer} than the number observations, the connections between observed components are sparse and each latent variable interacts with many observed components. The dependency structure is then learned by decomposing a matrix as a sum of low rank and sparse matrices.  These models however cannot handle our problem, as we consider a setting where the number of hidden variables is equal to the number of observed variables. 

Recently, in~\cite{geiger2015causal}, the authors considered a first order VAR model, which can capture number of latent variables up to the number of observed variables. Under this model, conditions such as non-Gaussianity and independent noise are utilized to learn the parameters; this is in spirit similar to Independent Component Analysis\cite{eriksson2004identifiability}. Our model is particularly tailored to the case when each observation has a corresponding latent component. The proposed algorithm for learning the model parameters utilizes the special structure of the model and does not need non-Gaussianity or independence to obtain consistent estimates. Furthermore, our model can incorporate more intricate memory due to random lags, whereas \cite{geiger2015causal} is limited to the first order Markov processes. We demonstrate the practical utility of our model by performing prediction on financial time series and climate datasets.

%% file: Sections/Model.tex
\section{Model and Problem Setup}

We consider a discrete time linear dynamical system where there are two types of variables: observed and latent. The unique feature of the proposed model is that the latent states influence the observed variables with \emph{random delays}. Let $z_t \in \mathbb{R}^p$ denote the vector of latent state variables and $x_t\in \mathbb{R}^p$ be the vector of observed variables. The system dynamics is described as
\begin{align}\label{Eq:Sys}
\begin{cases}
z_t = Az_{t-1} + Bx_{t-1} + v_t \\
x_t = Cz_{t-\Theta_t} + Dx_{t-1} + w_t.
\end{cases}
\end{align}
The random vectors $v_t \in \mathbb{R}^p$ and $w_t \in \mathbb{R}^p$ represent the state and observation noises, respectively, which are independent of each other and of the values of $z$ and $x$. Both of these noise sources are temporally white (uncorrelated from time step to time step) and spatially multivariate normally distributed with zero mean and covariance matrices denoted by $\Sigma_V$ and $\Sigma_W$, respectively.
The matrices $A$, $B$, $C$ and $D$ are of size $p \times p$.
The vector $\Theta_t\in \mathbb{Z}^p$ represents the delays incurred at the corresponding coordinates at time $t$. Each element of $\Theta_t$ is an integer-valued random variable, which is independent and identically distributed in $[0 \, , \, \theta_{\max}]$ according to some distribution, and is independent of everything else in the system. We denote the probability mass function of $\Theta$ by $q$, where $q_{\theta}=\Pr(\Theta=\theta)$.

Let $|y|$ denote the dimension of the vector $y$.
The model in \eqref{Eq:Sys} extends some existing models in literature. For example, by setting $\Theta_t=0$ and assuming $|z|\ll|x|$, the gene expression models introduced in \cite{rangel2005modeling,beal2005bayesian} can be recovered. The case where $\Theta_t=1$  and $|z|\ll|x|$ is considered by, for example, \cite{jalali2011learning}. In \cite{geiger2015causal}, a setting where $|z|\leq |x|$ and $\Theta=1$ is studied.
In this paper, we consider the setting where each observed variable has a corresponding latent state, i.e., $|z|=|x|$, and that each observed variable has a time-varying delay with respect to its latent state, implying that the matrix $C$ is diagonal (and invertible).

With $C$ being diagonal, without loss of generality, we can further restrict it to be the identity matrix. To see this, let $z'_t=Hz_t$, where $H$ is a nonsingular matrix. The model in~\eqref{Eq:Sys} can be written as:
\begin{align}\label{Eq:sysiden}
\begin{cases}
z'_t = HAH^{-1}z'_{t-1} + HBx_{t-1} + Hv_t \\
x_t = CH^{-1}z'_{t-\Theta_t} + Dx_{t-1} + w_t
\end{cases}
\end{align}
The matrix $H$ is required to be diagonal, because $CH^{-1}$ needs to be diagonal according to the model in~\eqref{Eq:Sys}. Note that the diagonality of $H$ enables us to write~\eqref{Eq:sysiden}, since we have $z'_{t-\Theta_t}=Hz_{t-\Theta_t}$. 
Moreover, similar to $v_t$, $Hv_t$ is multivariate normally distributed (although its covariance might be different, for which we do not have any requirement on). 
As a result, any coordinate transformation of latent variables as $z'_t=Hz_t$, with $H$ being a nonsingular diagonal matrix, will generate observations identically distributed with those of \eqref{Eq:Sys}.
Thus, without loss of generality, we can take $H$ to be $C^{-1}$, and the dynamical system to be
\begin{align}\label{Eq:Sys2}
\begin{cases}
z_t = Az_{t-1} + Bx_{t-1} + v_t \\
x_t = z_{t-\Theta_t} + Dx_{t-1} + w_t.
\end{cases}
\end{align}
\\

%% file: Sections/Equations.tex
\section{System Equations}\label{apx:analysis}

In this section, we derive equations in terms of covariances of the observations and system parameters. 

For simplicity, we assume that the system is initialized at origin, i.e., $x_0$ and $z_0$ are zero vectors. Thus, according to the model in.~\eqref{Eq:Sys2}, $x_t$ and $z_t$ are multivariate normally distributed with zero mean. 
The covariance of $\{x_t\}$ is defined as
$\Sigma_{X_i}=\Sigma_{X_t,X_{t-i}} = \mathbb{E}[x_tx^T_{t-i}]- \mathbb{E}[x_t]\mathbb{E}[x_{t-i}]^T$, which reduces to $\Sigma_{X_i}=\mathbb{E}[x_tx^T_{t-i}]$, assuming that the variables are zero mean. 
Let $\Sigma_{ZX_i}=\Sigma_{Z_t,X_{t-i}} = \mathbb{E}[z_tx^T_{t-i}]$ denote the cross-covariances of $x_t$ and $z_t$. Note that due to stationarity, we have $\Sigma_{ZX_i}=\Sigma_{Z_{t-j},X_{t-i-j}}$ for any integer $j$.
For notational convenience, we write $\Sigma_{X}=\Sigma_{X_0}$ and $\Sigma_{ZX}=\Sigma_{{ZX}_0}$.
Let $\tau_1$ and $\tau_2$ be vectors. We define the $(i,j)$-th element of $\Sigma_{Z_{\tau_1}X_{\tau_2}}$ as
$\Sigma_{Z_{\tau_1}X_{\tau_2}}(i,j) = \mathbb{E}[z_{t-\tau_1(i)}(i)x^T_{t-\tau_2(j)}(j)]=\mathbb{E}[z_{t}(i)x^T_{t-\tau_2(j)+\tau_1(i)}(j)]$.

Since $\Theta_t$ is i.i.d., we drop the subscript of $\Theta_t$ when it does not lead to confusion.
Let us denote $\Sigma_{UY_{\Theta}}=\mathbb{E}_{t,\Theta}[u_ty_{t-\Theta}^T]$, where $u$ and $y$ are vectors of time series and $\Theta$ is as defined in Equ.~\ref{Eq:Sys}. Note that $\Sigma_{UY_{\Theta}}$ is not random.
We use the following two Lemmas for deriving the equations and in proofs.
\begin{lem}
	$\Sigma_{UY_{\Theta}}=\sum_{\theta=0}^{\theta_{\max}} {q_{\theta}\Sigma_{UY_{\theta}}}.$
\end{lem}
\begin{proof}
	\begin{align}
	\nonumber \Sigma_{UY_{\Theta}}(i,j)&=\mathbb{E}_{t,\Theta}[u_t(i)y_{t-\Theta(j)}^T(j)]\\
	\nonumber &=\mathbb{E}_{\Theta}[\mathbb{E}_t[u_t(i)y^{T}_{t-\Theta(j)}(j) | \Theta(j)=\theta]]\\
	\nonumber &=\mathbb{E}_{\Theta}[\Sigma_{UY_{\theta}}(i,j)]\\
	\nonumber &=\sum_{\theta=0}^{\theta_{\max}} {q_{\theta}\Sigma_{UY_{\theta}}(i,j)}.
	\end{align}
\end{proof}

Let $y_t = u_{t-\Theta}$ and $u_{t} = Gu_{t-1}$, where $G$ is a $p\times p$ matrix. Note that $y_t \neq Gu_{t-\Theta-1}$, since each coordinate of $u$ is shifted independently at random, and therefore, the matrix $G$ applies to elements of $u$ in different times. The following Lemma shows that we still have $\mathbb{E}_t[y_{t}y_{t-k}^T]=\mathbb{E}_{t,\Theta}[Gu_{t-\Theta-1}y_{t-k}^T]=\mathbb{E}_{t,\Theta}[Gu_{t}y_{t-k+\Theta+1}^T]$.

\begin{lem}\label{lem:G}
	Let $y_t$, $u_t$ and $G$ be defined as above. We have 	$\mathbb{E}_t[y_{t}y_{t-k}^T]=G\Sigma_{UY_{k-\Theta-1}}$.
\end{lem}
\begin{proof}
	Let us denote $\Sigma=\mathbb{E}_t[y_{t}y_{t-k}^T]$. We have
	\begin{align}
	\nonumber \Sigma(i,:)&=\mathbb{E}_{t,\Theta}[u_{t-\Theta(i)}(i)y_{t-k}^T]\\
	\nonumber &=\mathbb{E}_{t,\Theta}[G(i,:)u_{t-\Theta(i)-1}y_{t-k}^T]\\
	\nonumber&=\mathbb{E}_{\Theta}[\mathbb{E}_t[G(i,:)u_{t-\Theta(i)-1} y_{t-k}^{T}| \Theta(i)=\theta]]\\
	\nonumber&=\mathbb{E}_{\Theta}[G(i,:)\Sigma_{UY_{k-\theta-1}}],
	\end{align}
	where $\Sigma(i,:)$ and $G(i,:)$ denote the $i$-th row of $\Sigma$ and $G$, respectively. Therefore, we have $\Sigma=\mathbb{E}_{\Theta}[G\Sigma_{UY_{k-\theta-1}}]=G\Sigma_{UY_{k-\Theta-1}}$.
\end{proof}

The following Lemma forms the basis of how we learn the parameters.
\begin{lem}
	\begin{equation}\label{Eq:opt1}
	\Resize{8cm}{
	\Sigma_{X_{i+1}}= (A+D)\Sigma_{X_i} - AD\Sigma_{X_{i-1}} + B \Sigma_{X_{i-\Theta}} \;\; , i\geq \max(\theta_{\max},1).}
	\end{equation}
\end{lem}
\begin{proof}
	We have
	\begin{align}
	&\Sigma_{X_{i}}= \Sigma_{ZX_{i-\Theta_t}} + D\Sigma_{X_{i-1}}, i\geq 1,\label{Eq:eq_xi}\\
	&\Resize{8cm}{\Sigma_{X_{i+1}}= A\Sigma_{ZX_{i-\Theta_t}} + B\Sigma_{X_{i-\Theta_t}} + \Sigma_{V_{\Theta_t}X_{i+1}} + D\Sigma_{X_{i}}, i\geq 0.}\label{Eq:eq_xi+1}
	\end{align}
	Equation~\ref{Eq:eq_xi} is obtained by multiplying $x_{t-i}^T$ to the second equation of the model~\eqref{Eq:Sys2} and taking the expectation. Since $w_t$ is independent of $x_{t-i}, i \geq 1$, we have $\Sigma_{W_tX_{t-i}}=0$.
	Equation~\ref{Eq:eq_xi+1} is also obtained by first multiplying $x_{t-i-1}^T$ to the second equation of the model and taking the expectation, and then replacing $\Sigma_{ZX_{i+1-\Theta_t}}$ from the first equation, using Lemma~\ref{lem:G}.
	To derive equations which are independent of the structure of $\Sigma_V$, we need to find $i$ such that $\Sigma_{V_{\Theta_t}X_{i+1}}$ is zero. We have
	$$\Sigma_{V_{\Theta_t}X_{i+1}}=\Sigma_{V_{\Theta_t}Z_{i+1+\Theta_{t-(i+1)}}}+\Sigma_{V_{\Theta_t}X_{i+2}}D^T.$$
	Recall that $v_t$ is independent of $z_{t-i}$ and $x_{t-i}$ for $i > 0$. Thus, for the first term to be zero, we must have $\Theta_t<i+1+\Theta_{t-(i+1)}$. Since $\Theta$ lies in $[0\,,\,\theta_{\max}]$ and is i.i.d., we must have $i>\theta_{\max}-1$.
	For the second term to be zero, we must have $\Theta_t < i+2$ and, thus, $i > \theta_{\max}-2$.
	Therefore, for $i\geq \theta_{\max}$, $\Sigma_{V_{\Theta_t}X_{i+1}}=0$. By replacing $\Sigma_{ZX_{i-\Theta_t}}$ in Equ.~\ref{Eq:eq_xi+1} from Equ.~\ref{Eq:eq_xi}, we obtain Equ.~\ref{Eq:opt1}.
\end{proof}

\begin{center}
	\begin{algorithm*}[ht]
		\caption{Learning the Sign-Sparsity pattern of $A$}
		\label{Alg1}
		\begin{algorithmic}[1]
			
			\INPUT{$\{x_t\}_{t=1,\dots,T}$}
			\OUTPUT{$S$: the sign-sparsity pattern of $A$}
			
			\State Determine $\theta_{\max}$, $\lambda_1$ and $\lambda_2$ using cross-validation.
			
			\State Compute the sample covariances $\hat{\Sigma}_{X_i}, i \in [0 \, , \, 2\theta_{\max}+2]$, from the data $\{x_t\}_{t=1,\dots,T}$.
			
			\State Let $B_1$, $K_1$ and $K_2$ be $p\times p$ matrices, where $p$ is the number of observed variables. Let $q$ be a vector with the length of $\theta_{\max}+1$. Solve the following optimization problem.
			\begin{align}
			\nonumber [{B_1}^*,{K_1}^*,{K_2}^*,q^*]=&\underset{B_1,K_1,K_2,q}{\text{arg\,min}} \sum_{i=\max(\theta_{\max},1)}^{2\theta_{\max}+1} \bigg\|B_1\hat{\Sigma}_{X_{i+1}}- K_1\hat{\Sigma}_{X_i} +K_2\hat{\Sigma}_{X_{i-1}} - \sum_{{\theta}=0}^{\theta_{\max}} {q_{\theta}\hat{\Sigma}_{X_{i-{\theta}}}} \bigg\|_{\text{Fro}}\\
			\nonumber & \mbox{s.t. } \;  B_1 \text{ is diagonal and positive definite},\\
			\nonumber & \;\;\;\;\;\;  \|K_1\|_1 \leq \lambda_1 \text{ and } \|K_2\|_1 \leq \lambda_2, \\
			\nonumber & \;\;\;\;\;\;  \sum_{{\theta}=0}^{\theta_{\max}}q_{\theta}=1, q_{\theta} \geq 0.
			\end{align}
			
			\State Return sign-sparsity pattern of $K_1^*$ as $S$.
			
		\end{algorithmic}
	\end{algorithm*}
\end{center}

%% file: Sections/Learning.tex
\section{Algorithm and Main Results}

Our goal is to learn system matrices $A$, $B$ and $D$ from a sequence of observations $\{x_t\}_{t=1,\dots,T}$, where $T$ is the length of the time series. In the following, we present the identifiability results and propose an algorithm for recovering the models parameters in a specific setting. 

\subsection{Identifiability of Model Parameters}

We first define the notion of generic and identifiable parameters.

\begin{defin}
A collection of parameters is said to be generic if each of the scalar parameters in the collection are chosen {\em independently} from a continuous distribution with a probability density. More formally, a collection of random variables is said to be generic if the probability measure of the collection is the product of probability measures on each scalar parameter and each probability measure is absolutely continuous with respect to the Lebesgue measure.
\end{defin} 

\begin{defin}[\cite{rangel2005modeling}]
Consider an observable random variable defined having probability distribution $P_{\pi}\in \{P_{\phi}:\boldsymbol{\phi} \in \Pi\}$, where the parameter space $\Pi$ is an open subset of the multi-dimensional Euclidean space. We say that this model is {\em identifiable} if the family $\{P_{\phi}:\boldsymbol{\phi} \in \Pi\}$ has the property that $P_{\pi}\equiv P_{\phi}$ if and only if $\boldsymbol{\pi}= \boldsymbol{\phi} \in \Pi$. In this parametric setting, we say that the parameter vector $\boldsymbol{\pi}$ is identifiable.
\end{defin} 
Identifiability property is important, for without it, it would be possible for different values of the parameters $\boldsymbol{\pi}$ to give rise to identically distributed observables, making the statistical problem of estimating $\boldsymbol{\pi}$ ill-posed \cite{rangel2005modeling}.

We assume that all the system parameters including matrices $A$, $B$, $D$, $\Sigma_V$, $\Sigma_W$, and the pmf $q$ are generic. 
Since there are as many latent variables are observed variables, not all entries of $A$, $B$ and $D$ are learnable even as $T$ grows to infinity. The following theorem states the identifiability of the system parameters in \eqref{Eq:Sys2} as the number of samples grows.

\begin{thm}\label{thm:param}
	Consider a dynamical system as defined in~\eqref{Eq:Sys2} and let $T \rightarrow \infty$. The following statements hold for all generic parameters of $A$, $B$, $D$, $\Sigma_V$, $\Sigma_W$ and $q$.
	\begin{itemize}
		\item $\theta_{\max}\in \{0,1\}$: We can identify $A+q_0B+D$ and $q_1B-AD$. Note that for $\theta_{\max}=0$, we have $q_0=1$ and $q_1=0$.
		\item $\theta_{\max}\geq 2$: We can identify $A+q_0B+D$, $q_1B-AD$, and the matrix $B$ up to a scalar multiple.
	\end{itemize}
\end{thm}
\begin{proof}
	Proof is provided in Appendix~\ref{apx:ProofThm1}.
\end{proof}

The following corollary is implied from Theorem~\ref{thm:param}.
\begin{cor}\label{cor:onesparse}
For all $\theta_{\max}$, if only one of the matrices $A$, $B$, or $D$ is known to be non-diagonal, then the non-diagonal entries of that matrix can be determined as $T \rightarrow \infty$.
\end{cor}

\subsection{Learning Algorithm}
Theorem \ref{thm:param} states that some combinations of system matrices are identifiable as the sample size grows to infinity, but does not provide an algorithm to estimate these matrices under finite sample regimes. 
In this section, we consider a specific setting and develop a consistent algorithm based on a convex optimization framework that recovers the true system parameters. 

In particular, we are interested in the setting where each latent variable interacts with only a few other latent variables, in addition to its corresponding observed variable. To isolate the interactions between latent variables, we assume that each observed variable has a causal relationship with only itself. Thus, the matrix $A$ is sparse and matrices $B$ and $D$ are diagonal.
Moreover, we assume that the matrix $B$ has positive diagonals, implying that each variable positively affects itself. Matrix $D$ also restricted to nonnegative (possibly zero) diagonals. 
In this setting, we are interested in learning the sign-sparsity pattern of $A$, since it captures the underlying causal graph of the process.
Through simulation studies with real-world data, we demonstrate the above setting is powerful enough to capture many practical phenomena.

Let $B'={B}^{-1}$, $K_1 = B'(A+D)$ and $K_2=B'AD$. Note that with $A$ being sparse, and $B$ and $D$ diagonal, we have that matrix $B'$ is diagonal and matrices $K_1$ and $K_2$ are sparse, with the same sign-sparsity pattern as $A$. We can write Equ.~\ref{Eq:opt1} as follows.
\begin{align}\label{Eq:opt2}
\Resize{8cm}{
	B'\Sigma_{X_{i+1}}- K_1\Sigma_{X_i} +K_2\Sigma_{X_{i-1}} - \Sigma_{X_{i-\Theta}}=0 \;\; , i\geq \max(\theta_{\max},1).}
\end{align}
Equation~\ref{Eq:opt2} provides a system of linear equations, from which we can obtain $K_1$ and $K_2$ and hence the sign-sparsity pattern of $A$. 
Algorithm~\ref{Alg1} gives the detailed description of learning the sign-sparsity pattern of $A$, and Lemma \ref{clm:alg} states its consistency.
Note that in algorithm, we return the sign-sparsity pattern of $K_1$, because the norm of $K_1$ is larger than the norm of $K_2$, and, as a result, its values are more reliable.

In algorithm~\ref{Alg1}, we use the sample covariance of the process $\{x_t\}$, calculated as $\hat{\Sigma}_{X_i}=\frac{1}{T} \sum_{t=i+1}^T x_t x^T_{t-i}$.
The $L_1$ norm--the sum of the absolute values of elements--of the matrix is denoted by $\|\cdot\|_1$, and used as a convex relaxation of the $L_0$ norm--the number of nonzero elements. Also, the Frobenius norm of a matrix, denoted by $\|\cdot\|_{\text{Fro}}$, is the square root of the sum of the squares of elements of the matrix.

\begin{lem}\label{clm:alg}
	Consider the model in~\ref{Eq:Sys2} with $A$ sparse, and $B$ and $D$ diagonal. Algorithm~\ref{Alg1} outputs the true sign-sparsity pattern of $A$ as $T\rightarrow \infty$.
\end{lem}
\begin{proof}
Algorithm~\ref{Alg1} finds the variables that minimize the error in the linear system of equations~\ref{Eq:opt2}. Since we have the exact estimate of the true covariances for $T\rightarrow \infty$, the objective function is at its minimum (zero) with arguments derived from the system parameters. 
Under the algorithm setting, $A$ is sparse, and $B$ and $D$ are diagonal (with non-negative diagonals). Thus, according to Corollary~\ref{cor:onesparse}, 
the sign-sparsity pattern of $A$ is identifiable as $T\rightarrow \infty$, and hence the algorithm outputs the true sign-sparsity pattern of $A$.
\end{proof}
Lemma \ref{clm:alg} states that the algorithm is \emph{consistent}; that is, Alg.~\ref{Alg1} outputs the same sign-sparsity pattern as the true $A$ as $T$ grows to infinity. A detailed analysis of the error rate as a function of $T$ is an important direction of future research.

%% file: Sections/Simulations.tex
\section{Experimental Results}

In this section, we study the performance of our method for three types of data: synthetic data, stock return data,  and climate data. In all three cases, we compare the performance of Alg.~\ref{Alg1} with the baseline Granger Lasso method~\cite{arnold2007temporal}. A general form of the Granger Lasso method can be written as follows.
\begin{align}
\Resize{8.9cm}{
\nonumber [A_1^*, \dots, A_L^*]=\underset{A_1, \dots, A_L}{\text{arg\,min}} \sum_{t=L+1}^{T} \bigg\| x_t - \sum_{i=1}^{L}A_ix_{t-i}\bigg\|_{\text{2}} + \sum_{i=1}^{L}\lambda_i \|A_i\|_1,}
\end{align}
where $\|\cdot\|_2$ denotes the $L_2$ norm--the square root of the sum of the squares of elements--of the vector, and for $i=1,\dots,L$, $A_i$ are $p\times p$ matrices and $\lambda_i$ are the penalty parameters. The dependency matrix $A_G^*$ is then obtained as $A_{G}^*=\sum_{i=1}^{L}A_i^*$.

For synthetic date, we also compare our method with a method \cite{jalali2011learning} proposed for first order vector autoregressive (VAR) model with hidden variables: 
\begin{align}\label{VAR}
\begin{cases}
z_t = Az_{t-1} + Bx_{t-1} + v_t \\
x_t = Cz_{t-1} + Dx_{t-1} + w_t.
\end{cases}
\end{align}
In this model, it is assumed that $|z|\ll|x|$, and $D$ is sparse and captures the underlying interactions among the variables. The dependency structure is learned by decomposing a matrix as a sum of low rank and sparse matrices, as follows.  
\begin{align}
\Resize{8.9cm}{
	\nonumber (D^*, L^*)=\underset{D, L}{\text{arg\,min}} \sum_{t=2}^{T} \bigg\| x_t - (D+L)x_{t-1}\bigg\|_{\text{2}} + \lambda_D \|D\|_1 + \lambda_L \|L\|_*,}
\end{align}
where $\|\cdot\|_*$ is the nuclear norm, i.e. sum of singular values, a convex surrogate for low-rank. 

\subsection{Synthetic Data}

The synthetic datasets are generated according to the model in~\eqref{Eq:Sys2} to study the performance of the Algorithm~\ref{Alg1} in recovering the true underlying temporal dependency graph of the latent variables. 
For matrix $A$, we generate a sparse matrix, where the sign of the nonzero elements are randomly assigned and the absolute value of the nonzero elements are generated uniformly at random. The diagonal elements of $B$ are also generated randomly according to a uniform distribution, and we set $D = 0$, which results in $K_2=0$. The matrices are properly scaled to ensure that the system is stable. Algorithm~\ref{Alg1} thus solves the following system of equations
$$B_1\Sigma_{X_{i+1}}- A_1\Sigma_{X_i} - \Sigma_{X_{i-\Theta}}=0 \;\; , i\geq \max(\theta_{\max},1),$$
where $A_1 = B^{-1}A$ and $B_1=B^{-1}$. Note that since $B$ is diagonal and its diagonal elements are positive, the sign-sparsity pattern of $A$ and $A_1$ are the same.

For each experiment, we generate $20$ random datasets and report the average performance on them.
Figure~\ref{fig:cmpr} shows the comparisons. For the first experiment, we consider $p=20$ number of variables, an average of five nonzero elements per each row of $A$, and $\theta_{\max}=5$. The results are reported for two different number of samples $T=5\times 10^4$ and $10^5$. For the second experiment, we consider $p=20$ number of variables, an average of two nonzero elements per each row of $A$, and $T=10^5$. The results are reported for two different values of $5$ and $10$ for $\theta_{\max}$.

In both experiments, we varied the sparsity threshold and reported the ROC curve which is the true positive rate versus the false positive rate of the sign-sparsity pattern of $A$.
Generally, with more number of samples and smaller maximum delay, we obtain a better estimation of the sign-sparsity pattern of $A$. Notably, our method significantly outperforms the Granger Lasso method and the method for first order VAR model with hidden variables \cite{jalali2011learning}, especially with small false positive rates. The results also indicate that the models in which $|z|\ll|x|$ cannot capture what our proposed dynamical system is able to model.

\begin{figure}[!t]
	\centering
	\begin{subfigure}[b]{0.4\textwidth}
		\centering
		\includegraphics[width=1\textwidth]{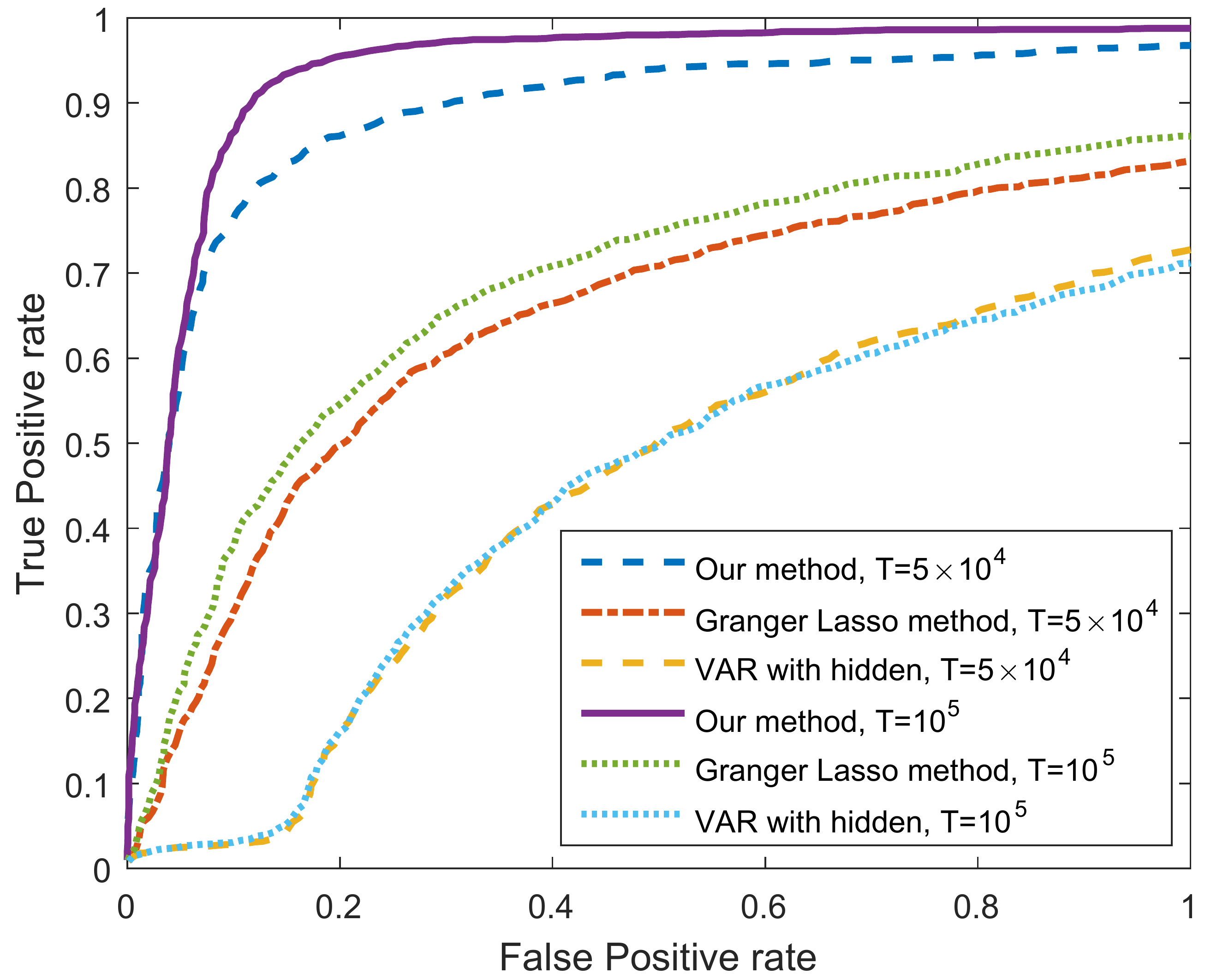}
		\caption{}
	\end{subfigure}
	\begin{subfigure}[b]{0.4\textwidth}
		\centering
		\includegraphics[width=1\textwidth]{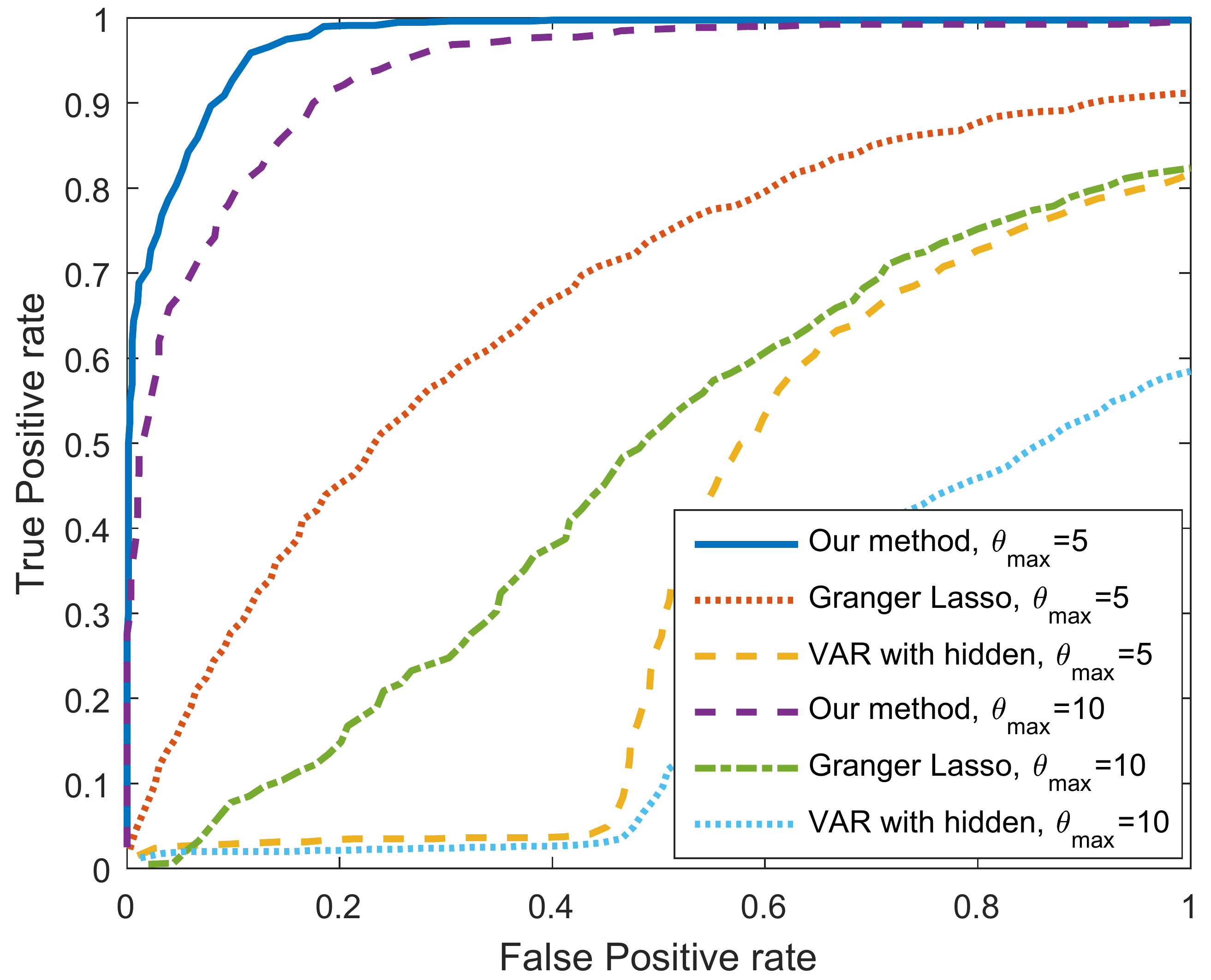}
		\caption{}
	\end{subfigure}
	\caption{Results on synthetic data for our proposed method, the Granger Lasso method and the vector autoregressive model with hidden variables~\ref{VAR}. (a) The results for $p=20$, on average five nonzero elements per each row of $A$, $\theta_{\max}=5$ and different number of samples. (a) The results for $p=20$, on average two nonzero elements per each row of $A$, $T=10^5$ and  different values for $\theta_{\max}$.}
	\label{fig:cmpr}
\end{figure}


\subsection{Stock Market Data} \label{sec:stock}

We take the end-of-the-day closing stock prices of $20$ companies for $11$ years in the period from the beginning of $2005$ to the end of $2015$ (roughly $250$ samples per year).
The companies are as follows: Apple, HP, IBM, Microsoft, Xerox, AT\&T, Verizon, Comcast, Oracle, Target, Walmart, Bank of America, Regions Financial, U.S. Bank, Wells Fargo, American Airlines, Caterpillar, Honeywell, International Paper and Weyerhaeuser. The data is collected from Google Finance website~\cite{googlefinance}. The data are normalized such that each variable has zero mean and unit variance.

Since the ground truth is not known, we follow the standard procedure of using prediction error to evaluate the algorithms. We train the models for the data from the first $10$ years and do the prediction for the final year. At each time, prediction is done for the $n$-th sample in future, where $n=1,2,\dots, 10$. For example, $n=3$ means that the stock price at the end of the day for the third day in the future is forecasted. We use $10$-fold cross-validation to tune the parameters.

For simplicity and due to limited training data, we set $D = dI$, and obtain $d$ from cross validation.
Cross validations result in $d=0.8$ and $\theta_{\max}=0$ for our method and the maximum delay of $L=1$ for the Granger Lasso method. \footnote{Note that $\theta_{\max}=0$ in our model corresponds to $L=1$ in Granger Lasso model, since in our model~\ref{Eq:Sys2} with $\theta_{\max}=0$, $z_t$ depends on $x_{t-1}$.} The results for the maximum delay is sensible, since changes tend to be incorporated into stock prices relatively quickly. We suspect more delays would be used if we consider prices at finer granularities (e.g. hourly or minute level).
Figure~\ref{fig:stock} shows the normalized mean square error for prediction results. It can be seen that our method outperforms the Granger Lasso method, especially for predicting farther samples in future.

\begin{figure}[!t]
	\centering
	\includegraphics[width=0.4\textwidth]{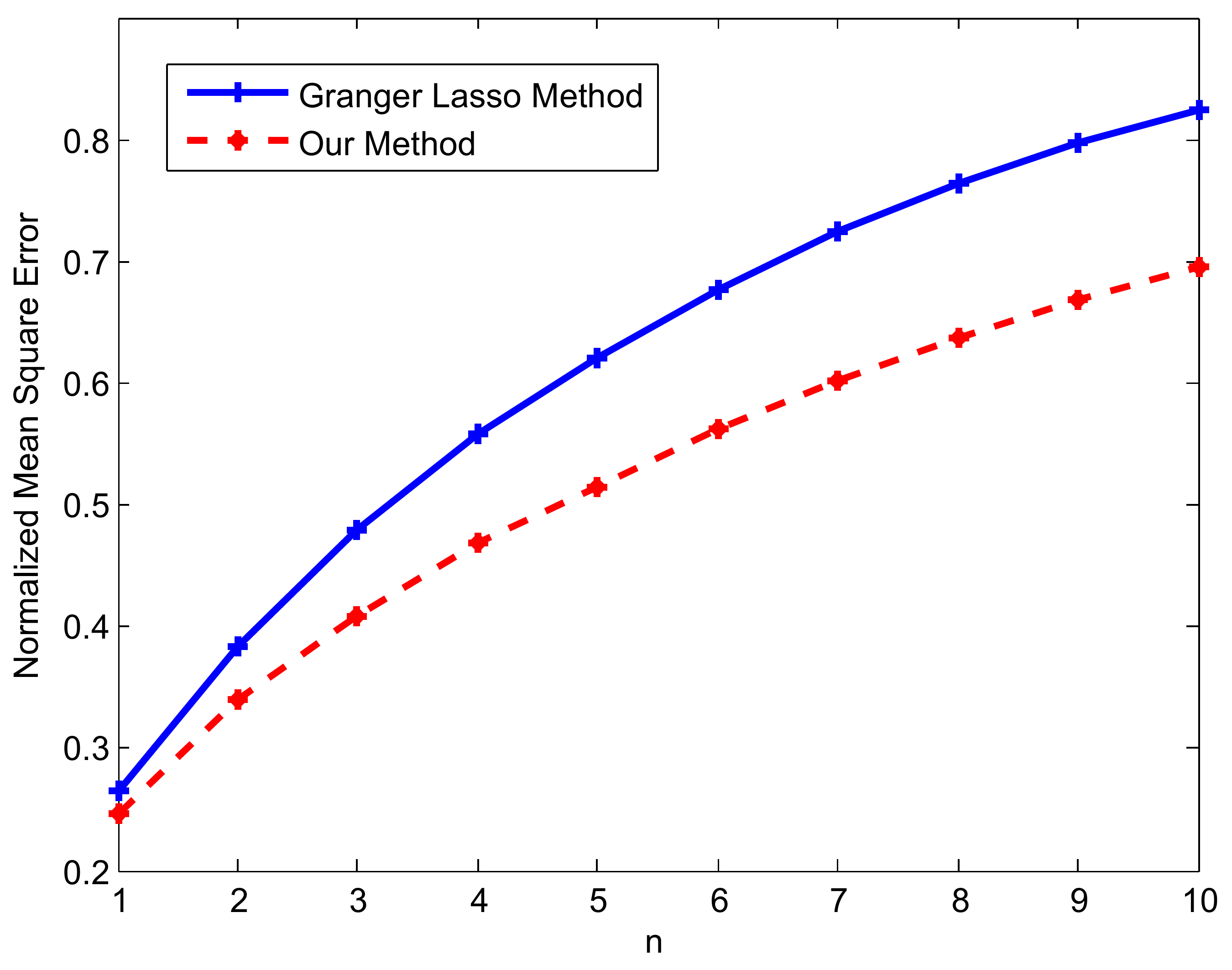}
	\caption{Stock market data results for our method and the Granger Lasso method. The results are shown for the normalize mean square error for predicting the $n$-th sample.}
	\label{fig:stock}
\end{figure}

\subsection{Climate Data} \label{sec:climate}

The Comprehensive Climate Dataset (CCDS)~\cite{climateData} is a collection of climate records of North America from~\cite{lozano2009spatial}. It contains monthly observations of climate variables spanning from $1990$ to $2002$. The observations were interpolated on a $2.5\times 2.5$ degree grid with $125$ observation locations.
The variables include Carbon-Dioxide, Methane, Carbon-Monoxide, Hydrogen, Wet Days, Cloud Cover, Vapor, Precipitation, Frost Days, Temperature, Temperature Range, Minimum Temperature, Maximum Temperature, Solar Radiation (Global Horizontal, Direct Normal, Global Extraterrestrial, Direct Extraterrestrial), and Ultraviolet (UV) radiation. In the following analysis, we omitted the UV variable because of significant missing data.

Similar to the stock market data, since the ground truth is not known, we use prediction accuracy as a measure of how well we can learn the underlying model. We train the model on data from the first $10$ years and do the prediction for the final three years. At each time, the prediction is done for the value of variables for the next month. The data are normalized across the variables and for different locations to have zero mean and unit variance. We use $10$-fold cross-validation to tune the parameters.

We set $D = dI$ in our model, where $d$ is also obtained from cross validation. Figure~\ref{fig:climate} shows the normalized mean square error for different values of the maximum delay $L$. Note that the maximum delay corresponds to $\theta_{\max}+1$ in our model. It can be seen that our method outperforms the Granger Lasso method in terms of the prediction accuracy. Also, for our method, $\theta_{\max}=3$ results in the smallest error, while for the Granger Lasso method, $L=1$ produces the best results. Thus, unlike the Granger model, our model suggests that the climate variables may depend on the value of other variables from several months ago, which is in accordance with the results of~\cite{lozano2009spatial}.

Moreover, notice that the matrix $D=dI$ captures the inter-sampling correlation and allows the latent variables to model the change occurred at time $t$. For example, in stock market data, the daily values are highly correlated, and we obtain $d=0.8$; whereas for climate data, we obtain $d=0$, which indicates that the monthly values of variables significantly change.

\begin{figure}[!t]
	\centering
	\includegraphics[width=0.4\textwidth]{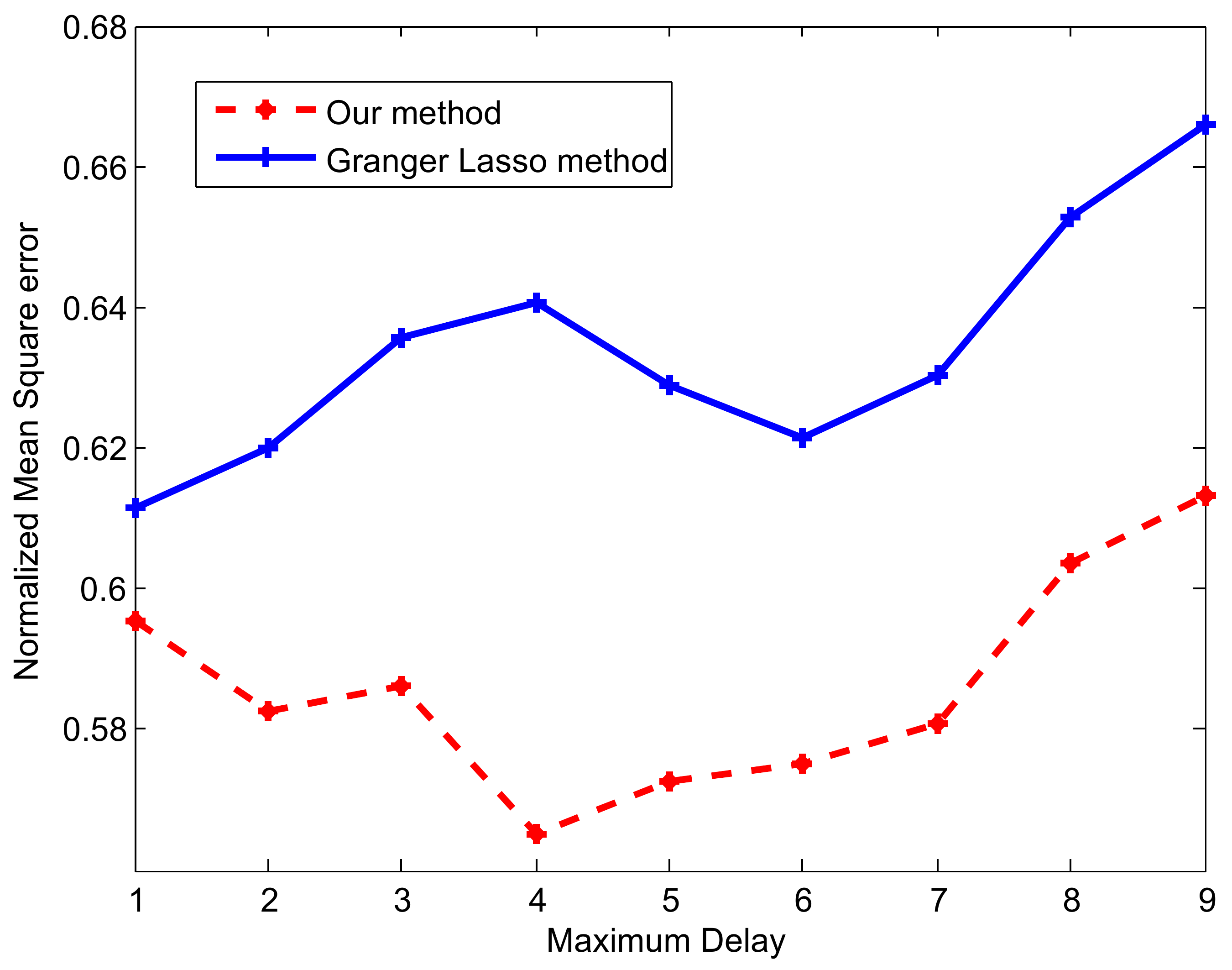}
	\caption{Climate data results for our method and the Granger Lasso method. The results are shown for the normalize mean square error for different values of the maximum delay.}
	\label{fig:climate}
\end{figure}

%% file: Sections/Conclusion.tex
\section{Conclusion} 
In this paper, we presented a method for learning the underlying dependency graph of time series, using a new model for temporal latent variables. We proposed an algorithm for learning the parameters and demonstrated its consistency. This study opens up several possible directions for future study. 

\begin{itemize}
\item  While we demonstrated consistency, more analysis is required in order to obtain sample complexity bounds for learning the structure and parameters of the model. Similarly, optimal predictors and missing value imputation algorithms need to be studied under this new model, which is especially interesting for the case when $\theta_{\max}$ is large.

\item We focused on temporal latent variables that simplify and capture temporal interactions more succinctly, while several existing works have focused on spatial latent variables, that mainly capture interactions between multiple variables \cite{jalali2011learning,geiger2015causal}. An interesting direction of future study is to combine these two models to obtain a single general model that can capture both. 

\item While we focused on the real-valued time-series model, binary and categorical time series are of equal practical interest. Developing algorithms under simple canonical models for these settings is of immediate interest. 

\item In this paper, we have worked on linear models; however, the memory of {\em stable} linear dynamical systems decays exponentially, with or without latent variables. While the latent variables can encode longer term linear interactions more succinctly, the model still suffers from the eventual memory decay. This is a problem well recognized in machine learning, and a class of algorithms based on Long Short-Term Memory (LSTM) \cite{hochreiter1997long} have demonstrated remarkable empirical performance. However, they lack theoretical backing, and developing this line of work to accommodate non-decaying memory can be an interesting direction of future work.

\end{itemize}

%% file: Sections/AppendixA.tex
\appendix
\section{Proofs of Theorem and Lemmas}\label{apx:ProofThm1}

We use the following Lemma for our analysis.
\begin{lem}
	The set of zeros of any nonzero multivarite polynomial of degree $n$ has Lebesgue measure zero on $\mathbb{R}^n$.\label{lm:lebes}\cite{geiger2015causal}
\end{lem}

We also use the equality $\mathsf{vec}(S_1RS_2)=(S_2^T  \otimes S_1)\mathsf{vec}(R)$ \cite{petersen2008matrix}, where $\mathsf{vec}$ is the vector-operator which stacks the columns of a matrix into a vector, $\otimes$ is the  Kronecker product, and $S_1$, $S_2$ and $R$ are square matrices with the same size. We denote $\mathsf{vec}(\Sigma_{(\cdot)})$ by $\mathsf{V}_{(\cdot)}$.


In the following, we prove Theorem~\ref{thm:param}.
 
\begin{lem}\label{lm:VV}
	For every $i \leq 0$, each element of $\Sigma_{VX_{i}}$ is a  multivariate polynomial in the elements of $A$, $B$, $D$, $\Sigma_V$ and the pmf $q$.
\end{lem}
\begin{proof}
	Using the equations of the model~\eqref{Eq:Sys2}, we can write
	\begin{align}
	&\Resize{8cm}{
	\Sigma_{VX_{-i}}=A\Sigma_{VZ_{\Theta+1-i}}+B\Sigma_{VX_{\Theta+1-i}}+D\Sigma_{VX_{1-i}} + \Sigma_{VV_{\Theta-i}}, i \geq 1,}\label{Eq:eq1-vx}\\
	&\Sigma_{VZ_{-i}}=A\Sigma_{VZ_{1-i}}+B\Sigma_{VX_{1-i}}, i \geq 1.\label{Eq:eq1-vz}
	\end{align}
	In Equations~\ref{Eq:eq1-vx} and \ref{Eq:eq1-vz}, the subscripts of the right-hand sides are greater then the subscripts of the left-hand sides. Moreover, note that, for $\Theta+1-i > 0$, we have $\Sigma_{VZ_{\Theta+1-i}}=\Sigma_{VX_{\Theta+1-i}}=0$; therefore, the subscripts are never positive. As a result, for every $i\geq 1$, by recursive use of Equations~\ref{Eq:eq1-vx} and \ref{Eq:eq1-vz}, each element of $\Sigma_{VX_{-i}}$ can be returned as a multivariate polynomial in the elements of $\Sigma_{VX_{0}}$, $\Sigma_{VZ_{0}}$, $\Sigma_V$, $A$, $B$, $D$ and the pmf $q$.
	Given that $\Sigma_{VX_{0}}=q_0\Sigma_{V}$ and $\Sigma_{VZ_{0}}=\Sigma_{V}$, we conclude that for every $i\geq 0$, each element of $\Sigma_{VX_{-i}}$ is a  multivariate polynomial in the elements of $\Sigma_V$, $A$, $B$,  $D$ and the pmf $q$.
\end{proof}


\begin{lem}\label{lm:VX}
	For every $i$, each element of $\Sigma_{X_i}$ can be expressed as a rational function, where both the nominator and the denominator are multivariate polynomials in elements of matrices $A$ and $B$, $D$, $\Sigma_{W}$, $\Sigma_{V}$ and pmf $q$.
\end{lem}
\begin{proof}
	We have the following equations
	\begin{align}\label{Eq:allsigma}
	\Resize{7.9cm}{\begin{cases}
	\Sigma_Z = A\Sigma_ZA^T + A\Sigma_{ZX}B^T + B\Sigma_{XZ}A^T + B\Sigma_X B^T + \Sigma_V \\
	\Sigma_{Z_i}=A\Sigma_{Z_{i-1}}+B\Sigma_{XZ_{i-1}}, 1 \leq i \leq \theta_{\max}\\
	\Sigma_{ZX_i}=A\Sigma_{ZX_{i-1}}+B\Sigma_{X_{i-1}}, 1 \leq i \leq \theta_{\max}\\
	\Sigma_{ZX_{-i}}=A\Sigma_{ZX_{-i-1}}+B\Sigma_{X_{-i-1}}+\Sigma_{VX_{-i}}, 0 \leq i \leq \theta_{\max}-1\\
	\Sigma_{X}=\sum_{\theta=0}^{\theta_{\max}}q_{\theta}\Sigma_{ZX_{\theta}} + D\Sigma_{X_{-1}} + \Sigma_W \\
	\Sigma_{X_i}=\sum_{\theta=0}^{\theta_{\max}}q_{\theta}\Sigma_{ZX_{i-\theta}} + D\Sigma_{X_{i-1}}, 1 \leq i \leq \theta_{\max} \\
	\Sigma_{ZX}=\sum_{\theta=0}^{\theta_{\max}}q_{\theta}\Sigma_{Z_{\theta}} + \Sigma_{ZX_1}D^T
	\end{cases}}
	\end{align}
	where the first four equations are derived from the first equation of the model~\eqref{Eq:Sys2} and  the last three equations are derived from the second equation of the model. Note that, for every $i$, $\Sigma_{X_{-i}}=\Sigma_{X_{i}}^T$. By converting the equations in~\ref{Eq:allsigma} to vector from, we have
	\begin{align}\label{Eq:allvec}
	\Resize{7.9cm}{\begin{cases}
	(I - A\otimes A)\mathsf{V}_Z = (B\otimes A)\mathsf{V}_{ZX} + (A\otimes B)\mathsf{V}_{XZ} + (B\otimes B)\mathsf{V}_X + \mathsf{V}_V \\ \mathsf{V}_{Z_i}=(I\otimes A)\mathsf{V}_{Z_{i-1}}+(I\otimes B)\mathsf{V}_{XZ_{i-1}}, 1 \leq i \leq \theta_{\max}\\
	\mathsf{V}_{ZX_i}=(I\otimes A)\mathsf{V}_{ZX_{i-1}}+(I\otimes B)\mathsf{V}_{X_{i-1}}, 1 \leq i \leq \theta_{\max}\\
	\mathsf{V}_{ZX_{-i}}=(I\otimes A)\mathsf{V}_{ZX_{-i-1}}+(I\otimes B)\mathsf{V}_{X_{-i-1}}+\mathsf{V}_{VX_{-i}}, 0 \leq i \leq \theta_{\max}-1\\
	\mathsf{V}_{X}=\sum_{\theta=0}^{\theta_{\max}}q_{\theta}\mathsf{V}_{ZX_{\theta}} + (I\otimes D)\mathsf{V}_{X_{-1}} + \mathsf{V}_W \\
	\mathsf{V}_{X_i}=\sum_{\theta=0}^{\theta_{\max}}q_{\theta}\mathsf{V}_{ZX_{i-\theta}} + (I\otimes D)\mathsf{V}_{X_{i-1}}, 1 \leq i \leq \theta_{\max} \\
	\mathsf{V}_{ZX}=\sum_{\theta=0}^{\theta_{\max}}q_{\theta}\mathsf{V}_{Z_{\theta}} + (D\otimes I)\mathsf{V}_{ZX_1}. \\
	\end{cases}}
	\end{align}

	According to Lemma~\ref{lm:VV}, each element of $\Sigma_{VX_{-i}}, i \geq 0,$ can be returned as a  multivariate polynomial in the elements of the system parameters. Also, note that since $\Sigma_{XZ}=\Sigma_{ZX}^T$, $\mathsf{V}_{XZ}$ is a permutation of $\mathsf{V}_{ZX}$.
	Therefore, Equ.~\ref{Eq:allvec} forms a linear system of equations with $p^2(4\theta_{\max}+3)$ number of equations and the same number of variables in elements of $\mathsf{V}_{X_i}, i \in [0 \,,\, \theta_{\max}] \cup \{1\},$ $\mathsf{V}_{Z_i}, i \in [0 \,,\, \theta_{\max}],$ and $\mathsf{V}_{ZX_i}, i \in [-\theta_{\max} \,,\, \theta_{\max}] \cup \{1\}$.
	As a result, each element of $\Sigma_{X_i}, 0\leq i \leq \theta_{\max},$ can be solved as a rational function, where both the nominator and the denominator are multivariate polynomials in elements of system parameters. Using Equ.~\ref{Eq:opt1}, $\Sigma_{X_i}, i > \theta_{\max},$ can be also expressed as rational functions with both the nominator and denominator as multivariate polynomials. Therefore, the proof is complete.
\end{proof}


Let $M^{(k)}$ be a block Toeplitz matrix as follows.
\begin{align}\label{Mk}
M^{(k)}=
\begin{bmatrix}
\Sigma_{X_{k+1}} & \Sigma_{X_{k+2}} & \dots & \Sigma_{X_{2k+1}} \\
\Sigma_{X_{k}} & \Sigma_{X_{k+1}} & \dots & \Sigma_{X_{2k}} \\
\vdots & \vdots & \ddots & \vdots \\
\Sigma_{X_{1}} & \Sigma_{X_{2}} & \dots & \Sigma_{X_{k+1}}
\end{bmatrix}.
\end{align}

We consider two cases of $\theta_{\max} \in \{0,1\}$ and $\theta_{\max} \geq 2$.

\noindent \textbf{Case 1:} $\bm{\theta}_{\max} \in \mathbf{\{0,1\}}$

In this case, Equ.~\ref{Eq:opt2} can be written as ${B}^{-1}\Sigma_{X_{i+1}}= (K_1+q_0I)\Sigma_{X_i} +(-K_2+q_1I)\Sigma_{X_{i-1}}, i\geq 1$, and therefore
\begin{align}\label{Eq:theta0eq}
\Resize{7.9cm}{
\Sigma_{X_{i-1}}= (-K_2+q_1I)^{-1}{B}^{-1}\Sigma_{X_{i+1}} -(-K_2+q_1I)^{-1}(K_1+q_0I)\Sigma_{X_{i}}, i\geq 1.}
\end{align}
Note that for $\theta_{\max}=0$, we have $q_0=1$ and $q_1=0$. Writing Equ.~\ref{Eq:theta0eq} For $i \in \{1,2\}$, we get
\begin{align}\label{Eq:identif0}
\Resize{7.9cm}{\begin{bmatrix}
\Sigma_{X} & \Sigma_{X_{1}}
\end{bmatrix}=
\begin{bmatrix}
(-K_2+q_1I)^{-1}{B}^{-1} & -(-K_2+q_1I)^{-1}(K_1+q_0I)
\end{bmatrix}M^{(1)}.}
\end{align}


\begin{lem}\label{lm:inst0}
	Following model~\eqref{Eq:Sys2} with $\theta_{\max} \in \{0,1\}$, there is an instance of parameters for which matrix $M^{(1)}$ defined in Equ.~\ref{Mk} is full-rank.
\end{lem}
\begin{proof}
	Let $A=D=\frac{I}{2}$, $B=0$, and $q_0=1$ and $q_1=0$. Using the model equations, we obtain $\Sigma_{X}=\frac{80}{27}\Sigma_V + \frac{4}{3}\Sigma_W$, $\Sigma_{X_1}=\frac{64}{27}\Sigma_V + \frac{2}{3}\Sigma_W$, $\Sigma_{X_2}=\frac{44}{27}\Sigma_V + \frac{1}{3}\Sigma_W$ and $\Sigma_{X_3}=\frac{28}{27}\Sigma_V + \frac{1}{6}\Sigma_W$.
	By setting $\Sigma_V=I$ and $\Sigma_W=I$, the corresponding instance of $M^{(1)}$ is full-rank.
\end{proof}


\begin{thm}\label{thm:fullrank0}
	Following model~\eqref{Eq:Sys2} with $\theta_{\max} \in \{0,1\}$, the matrix $M^{(1)}$ defined in Equ.~\ref{Mk} is full-rank for generic system parameters $A$, $B$, $D$, $\Sigma_V$, $\Sigma_W$ and $q$.
\end{thm}
\begin{proof}
	According to Lemma~\ref{lm:VX}, each element of $\Sigma_{X_i}$ can be written as a rational function, where both the nominator and the denominator are multivariate polynomials in elements of the system parameters $A$, $B$, $D$, $\Sigma_V$, $\Sigma_W$ and $q$. Therefore, we can conclude that $\det(M)=\frac{P}{Q}$, where $P$ and $Q$ are multivariate polynomials in elements of the system parameters.
	According to Lemma~\ref{lm:inst0}, these multivarite polynomials are nonzero, because there exists an instance which results in a nonzero value for $\det(M)$. Since the Lebesgue measure of roots of a nonzero multivariate polynomial is zero, we conclude that $\det(M)$ is nonzero for generic parameters.
\end{proof}


\begin{lem}\label{lm:notfull0}
	Following model~\eqref{Eq:Sys2} with $\theta_{\max} \in \{0,1\}$, for $k \geq 2$, the matrix $M^{(k)}$ defined in Equ.~\ref{Mk} is not full-rank.
\end{lem}
\begin{proof}
	According to Equ.~\ref{Eq:theta0eq}, $\Sigma_{X_i}, i\geq 3,$ can be written in terms of $\Sigma_{X_{i-1}}$ and $\Sigma_{X_{i-2}}$. Therefore, for $k \geq 2$, the first row of $M^{(k)}$ is a linear function of other rows, and, as a result, the matrix is not full-rank.
\end{proof}


\noindent \textbf{Case 2:} $\bm{\theta}_{\max} \geq \mathbf{2}$

By expanding Equ.~\ref{Eq:opt2}, we can write
\begin{align}\label{Eq:mat1}
\nonumber \Sigma_{X_{i-\theta_{\max}}} = &\frac{1}{q_{\theta_{\max}}}{B}^{-1}\Sigma_{X_{i+1}} \\
\nonumber &-\frac{q_0I+K_1}{q_{\theta_{\max}}}\Sigma_{X_i} - \frac{q_1I - K_2}{q_{\theta_{\max}}}\Sigma_{X_{i-1}} \\
&- \frac{1}{q_{\theta_{\max}}}\sum_{\theta=2}^{\theta_{\max}-1}q_{\theta}\Sigma_{X_{i-\theta}}, i\geq \theta_{\max}.
\end{align}
Writing Equ.~\ref{Eq:mat1} For $i \in [\theta_{\max} \, ,\, 2\theta_{\max}]$, we get
\begin{align}\label{Eq:identif}
\nonumber &\begin{bmatrix}
\Sigma_{X_{0}} & \Sigma_{X_{1}} & \dots & \Sigma_{X_{\theta_{\max}}}
\end{bmatrix}=\\
&\Resize{7.9cm}{
\begin{bmatrix}
\frac{{B}^{-1}}{q_{\theta_{\max}}} & -\frac{q_0I+K_1}{q_{\theta_{\max}}} & - \frac{q_1I - K_2}{q_{\theta_{\max}}} & -\frac{q_2 I}{q_{\theta_{\max}}} & \dots & -\frac{q_{\theta_{\max}-1}I}{q_{\theta_{\max}}}
\end{bmatrix}M^{(\theta_{\max})}.}
\end{align}


\begin{lem}\label{lm:inst}
	Following model~\eqref{Eq:Sys2} with $\theta_{\max} \geq 2$, there is an instance of parameters for which matrix $M^{(\theta_{\max})}$ defined in Equ.~\ref{Mk} is full-rank.
\end{lem}
\begin{proof}
	Let $A=D=0$, $B=bI, 0<b<1$, $\Sigma_V=0$, $\Sigma_W=I$, and $q$ be a uniform distribution in $[0 \, , \, \theta_{\max}]$. Using the model equations~\eqref{Eq:Sys2}, we get
	\begin{align}\label{eq:ins}
	x_t=bx_{t-\Theta-1}+w_t.
	\end{align}
	Since variables are independent with the same distribution, all matrices $\Sigma_{X_i}, i\in [1 \,, \, 2\theta_{\max}+1],$ are multiples of identity. Therefore, we can write $M^{(\theta_{\max})}=M_1 \otimes I$, where $M_1$ is the corresponding matrix for one variable, as follows
	\begin{align}
	\nonumber M_1=
	\begin{bmatrix}
	\sigma_{X_{\theta_{\max}+1}} & \sigma_{X_{\theta_{\max}+2}} & \dots & \sigma_{X_{2\theta_{\max}+1}} \\
	\sigma_{X_{\theta_{\max}}} & \sigma_{X_{\theta_{\max}+1}} & \dots & \sigma_{X_{2\theta_{\max}}} \\
	\vdots & \vdots & \ddots & \vdots \\
	\sigma_{X_{1}} & \sigma_{X_{2}} & \dots & \sigma_{X_{\theta_{\max}+1}}
	\end{bmatrix},
	\end{align}
	where $\sigma$ denotes the variance. Since $\mathrm{rank}(M^{(\theta_{\max})})=p \; \mathrm{rank}(M_1)$, we only need to show that matrix $M_1$ is full-rank.
	Using Equ.~\ref{eq:ins}, we can write
	\begin{align}\label{eq:ins2}
	\begin{cases}
	\sigma_{X}=bq_0\sum_{j=0}^{\theta_{\max}}\sigma_{X_{j+1}}+1 \\
	\sigma_{X_i}=bq_0\sum_{j=0}^{\theta_{\max}}\sigma_{X_{j+1-i}}, 1 \leq i \leq \theta_{\max}+1.
	\end{cases}
	\end{align}
	where $q_0=\frac{1}{\theta_{\max}+1}$. By solving the system of equations in \ref{eq:ins2}, we obtain
	\begin{align}
	\nonumber\begin{cases}
	\sigma_{X}=\frac{(1-b)\theta_{\max}+1}{(1-b)\theta_{\max}+(1-b^2)} \\
	\sigma_{X_i}=\frac{b}{(1-b)\theta_{\max}+(1-b^2)}, 1 \leq i \leq \theta_{\max}+1.
	\end{cases}
	\end{align}
	Also, using the following equation
	\begin{equation}\label{eq_ins}
	\nonumber \scalebox{0.95}{
	$\begin{bmatrix}
	\sigma_{X_{0}} & \sigma_{X_{1}} & \dots & \sigma_{X_{\theta_{\max}}}
	\end{bmatrix}=
	\begin{bmatrix}
	\frac{b^{-1}}{q_0} & -1 & -1 & \dots & -1
	\end{bmatrix}M_1,$}
	\end{equation}
	we solve $\sigma_{X_{\theta_{\max}+2}}=b\sigma_{X_{\theta_{\max}+1}}$, and, since $b<1$, we have $\sigma_{X_{\theta_{\max}+2}}\neq \sigma_{X_{\theta_{\max}+1}}$. Now, we compute the determinant of $M_1$ by doing the matrix column operation. More specifically, we replace $m_i$  with $m_i-m_{i+1},1 \leq i \leq \theta_{\max},$ where $m_i$ is the $i$-th column of $M_1$. The new matrix is an upper triangular matrix, as follows
	\begin{equation}\scalebox{0.61}{
	\nonumber $M'_1=
	\begin{bmatrix}
	\sigma_{X_{\theta_{\max}+1}}-\sigma_{X_{\theta_{\max}+2}} & \sigma_{X_{\theta_{\max}+2}}-\sigma_{X_{\theta_{\max}+3}} & \dots & \sigma_{X_{2\theta_{\max}}}-\sigma_{X_{2\theta_{\max}+1}} & \sigma_{X_{2\theta_{\max}+1}} \\
	0 & \sigma_{X_{\theta_{\max}+1}}-\sigma_{X_{\theta_{\max}+2}} & \dots & \sigma_{X_{2\theta_{\max}-1}}-\sigma_{X_{2\theta_{\max}}} & \sigma_{X_{2\theta_{\max}}} \\
	\vdots & \vdots & \ddots & \vdots & \vdots \\
	0 & 0 & \dots & \sigma_{X_{\theta_{\max}+1}}-\sigma_{X_{\theta_{\max}+2}} & \sigma_{X_{\theta_{\max}+2}} \\
	0 & 0 & \dots & 0 & \sigma_{X_{\theta_{\max}+1}}
	\end{bmatrix}.$}
	\end{equation}
	Since all diagonals are nonzero, the determinant of $M'_1$ is nonzero, and, as a result, matrix $M'_1$ and also $M_1$ are full-rank.
\end{proof}


\begin{thm}\label{thm:fullrank}
	Following model~\eqref{Eq:Sys2} with $\theta_{\max} \geq 2$, the matrix $M^{(\theta_{\max})}$ defined in Equ.~\ref{Mk} is full-rank for generic system parameters $A$, $B$, $D$, $\Sigma_V$, $\Sigma_W$ and $q$.
\end{thm}
\begin{proof}
	Proof is similar to the proof of Theorem~\ref{thm:fullrank0}.
\end{proof}


\begin{lem}\label{lm:notfull}
	Following model~\eqref{Eq:Sys2} with $\theta_{\max} \geq 2$, the matrix $M^{(k)}$ defined in Equ.~\ref{Mk} is not full-rank, for $k \geq \theta_{\max}+1$.
\end{lem}
\begin{proof}
	According to Equ.~\ref{Eq:opt2}, $\Sigma_{X_i}, i\geq \theta_{\max}+2,$ can be written in terms of
	$j \in [i-1-\theta_{\max} \, , \, i-1]$. Therefore, for $k \geq \theta_{\max}+1$, the first row of $M^{(k)}$ is a linear function of other rows, and, as a result, the matrix is not full-rank.
\end{proof}


For $\theta_{\max} \in \{0,1\}$, according to Theorem~\ref{thm:fullrank0}, $M^{(1)}$ is full-rank for generic parameters, and according to Lemma~\ref{lm:notfull0}, $M^{(k)}$ is not full-rank for $k > 1$.
Also, for $\theta_{\max} \geq 2$, according to Theorem~\ref{thm:fullrank}, $M^{(\theta_{\max})}$ is full-rank for generic parameters, and according to Lemma~\ref{lm:notfull}, $M^{(k)}$ is not full-rank for $k > \theta_{\max}$. Therefore, we have the following Corollary.
\begin{cor}\label{cor:theta}
Let $M^{(k)}$ be as defined in Equ.~\ref{Mk}. The maximum $k$ for which $M^{(k)}$ is full-rank is $\max(\theta_{\max}, 1)$.
\end{cor}

Now, we are ready to prove Theorem~\ref{thm:param}. Recall that $K_1 = {B}^{-1}(A+D)$ and $K_2={B}^{-1}AD$. We have the following cases.
\begin{itemize}
	\item $\theta_{\max}\in \{0,1\}$: 
	According to Corollary~\ref{cor:theta}, we can identify that $\theta_{\max} \in \{0,1\}$ and thus we can form Equ.~\ref{Eq:identif0}. Since $M^{(1)}$ is full rank, using this equation,
	we can identify the matrices $(-K_2+q_1I)^{-1}{B}^{-1}$ and $(-K_2+q_1I)^{-1}(K_1+q_0I)$. This is equivalent to identifying  $A+q_0B+D$ and $q_1B-AD$. Note that, in this case, we cannot determine $\theta_{\max}$. Also, note that for $\theta_{\max}=0$, we have $q_0=1$ and $q_1=0$.

	\item $\theta_{\max}\geq 2$: According to Corollary~\ref{cor:theta}, we can identify $\theta_{\max}$ and thus we can form Equ.~\ref{Eq:identif}. Since $M^{(\theta_{\max})}$ is full rank, using this equation, 
	we can identify ${B}^{-1}/q_{\theta_{\max}}$, $(q_0I+K_1)/q_{\theta_{\max}}$, $(q_1I - K_2)/q_{\theta_{\max}}$ and $q_i/q_{\theta_{\max}} , i \in [2 , \theta_{\max}-1]$.
	This is equivalent to identifying $B$ up to a scalar multiple, $A+q_0B+D$ and $q_1B-AD$.
\end{itemize}
